\documentclass{article}

\usepackage{amsmath,amssymb,graphicx,epstopdf,color}
\usepackage{url,algorithm,algorithmic}

\usepackage{amsthm,vmargin}

\setpapersize{A4}

\newcommand{\g}[1]{\boldsymbol{#1}}

\newcommand{\I}[1]{\mathbf{1}_{#1}}

\renewcommand{\H}[0]{\mathcal{H}}

\newcommand{\Q}[0]{\mathcal{Q}} 
\newcommand{\QR}[0]{\mathbb{Q}} 

\renewcommand{\O}[0]{\mathcal{O}}

\newcommand{\sign}[0]{\mbox{sign}} 
\newtheorem{theorem}{Theorem}
\newtheorem{problem}{Problem}
\newtheorem{assumption}{Assumption}

\newtheorem{proposition}{Proposition}
\newtheorem{lemma}{Lemma}
\newtheorem{corollary}{Corollary}
\newtheorem{remark}{Remark}

\newcommand{\argmax}{\operatornamewithlimits{argmax}}
\newcommand{\argmin}{\operatornamewithlimits{argmin}}

\begin{document}

\title{\bf On the complexity of switching linear regression} 
\author{Fabien Lauer\medskip\\\small Universit\'e de Lorraine, CNRS, LORIA, UMR 7503, F-54506 Vand\oe{}uvre-l\`es-Nancy, France}
\maketitle
\begin{abstract}
This technical note extends recent results on the computational complexity of globally minimizing the error of piecewise-affine models to the related problem of minimizing the error of switching linear regression models. In particular, we show that, on the one hand the problem is NP-hard, but on the other hand, it admits a polynomial-time algorithm with respect to the number of data points for any fixed data dimension and number of modes. 
\end{abstract}

\section{Introduction}
\label{sec:intro}

Hybrid system identification aims at estimating a model of a system switching between different operating modes from input-output data and is typically setup as a piecewise-affine (PWA) or switching regression problem (see \cite{Paoletti07,Garulli12} for an overview). 
The present paper focuses on the issue of deterministically obtaining a global solution to the switching regression problem. 
In particular, we are interested in the rather theoretical question of the existence of an algorithm for this problem with a reasonable (i.e., polynomial) time complexity. Therefore, we will concentrate the discussion on computational complexity issues
 under the classical model of computation known as a Turing machine \cite{Garey79}. In this framework, the time complexity of a problem is the lowest time complexity of an algorithm solving any instance of that problem, where the time complexity of an algorithm is the maximal number of steps occurring in the computation of the corresponding Turing machine program.  

Let $[n] = \{1,\dots,n\}$ denote the set of integers from 1 to $n$, $\g x_i\in\QR^d$ the regression vector of index $i$ (possibly built from lagged inputs and outputs of a dynamical system) and $y_i\in\QR$ the corresponding output. We consider the estimation of the parameters $\{\g w_j\}_{j=1}^n\subset\QR^d$ of an arbitrarily switching linear model $y_i = \g w_{q_i}^T\g x_i + v_i$, where $q_i\in\Q=[n]$ stands for the active mode at index $i$ and $v_i\in\QR$ is a noise term. 
We assume that the mode $q_i$ is independent of $\g x_i$, that a data set $\{(\g x_i, y_i)\}_{i=1}^N\subset \QR^d\times \QR$ of size $N$ significantly larger than the dimension $d$ is available, 
and that the number of modes $n$ is given. 
We concentrate on the most common approach minimizing the prediction error over the variables to be estimated, here the classification of the points into modes, i.e., $\g q\in\Q^N$, and the parameter vectors, $\{\g w_j\}_{j=1}^n$. 
Specifically, we formulate the problem in terms of a loss function $\ell : \QR\rightarrow \QR^+$, 
assumed to be 
computable in polynomial time and to satisfy
\begin{equation}\label{eq:loss}
	\begin{cases}
	\ell(0) = 0,\\
	\forall e\in\QR,\ \ell(-e) = \ell(e),\\	
	\forall (e,e^\prime)\in\QR^2,\ \ell(e) < \ell(e^\prime) \Leftrightarrow |e| < |e^\prime|.
	\end{cases}
\end{equation}
\begin{problem}[Switching linear regression]\label{pb:min}
Given a data set $\{(\g x_i, y_i)\}_{i=1}^N \subset\QR^d\times \QR$ and an integer $n\in[2,N/d]$, find a global solution to 
\begin{equation}
	\min_{\{\g w_j\in\QR^d\}_{j=1}^n, \g q\in\Q^N}\ \frac{1}{N} \sum_{i=1}^N   \ell (y_i- \g w_{q_i}^T\g x_i ).
\end{equation}
\end{problem}
Other equivalent formulations based on mixed-integer programming with binary variables encoding $\g q$ or on continuous optimization can be found in \cite{Paoletti07,Lauer11a} and a number of heuristics subject to local minima \cite{Lauer13a,Lauer14a} or only optimal under specific conditions \cite{Vidal03,Bako11} have been proposed. 

Problem~\ref{pb:min} can be solved explicitly with respect to (wrt.) $\g q$ for fixed $\{\g w_j\}_{j=1}^n$ by assigning each point to the model with minimum error as
\begin{equation}\label{eq:classif}
	q_i \in \argmin_{j\in\{1,\dots,n\}} \ell(y_i - \g w_j^T \g x_i ),\quad i=1,\dots,N.
\end{equation}
Conversely, Problem~\ref{pb:min} can be solved wrt. to the $\g w_j$'s for fixed $\g q$ as $n$ independent linear regression subproblems
\begin{equation}\label{eq:wj}
	\min_{\g w_j \in\QR^{d} }\ \sum_{i\in\{k  : q_k = j\}} \ell(y_i- \g w_j^T \g x_i ),\quad j=1,\dots,n.
\end{equation}
Thus, two global optimization approaches can be readily formulated. 
The first one tests all possible classifications $\g q$ and solves the problem wrt. the $\g w_j$'s for each of them. But, this leads to $n \times n^N$ linear regression subproblems~\eqref{eq:wj} and quickly becomes intractable when $N$ increases. 
The second approach applies a continuous global optimization strategy to directly estimate $\{\g w_j\}_{j=1}^n$ under the optimal classification rule~\eqref{eq:classif}. However, global optimality cannot be guaranteed without constraints  on the $\g w_j$'s such as box bounds. And even so, the complexity remains exponential in the number of variables $nd$, for instance for a grid search to obtain a solution with an error that is only guaranteed to be close to the global optimum in finite time.

These straightforward observations illustrate the difficulty of the problem, which is here quantified more formally. In particular, we prove in Sect.~\ref{sec:nphard} that Problem~\ref{pb:min} is NP-hard. Nonetheless, we also show in Sect.~\ref{sec:exact} that the problem can be solved in polynomial time wrt. the number of data for fixed $n$ and $d$. This result is obtained by generalizing ideas developed for PWA systems in \cite{Lauer15} and by deriving for the first time a clear connection between switching regression and linear classification.

\section{NP-hardness}
\label{sec:nphard}

In computational complexity, an NP-hard problem is one that is at least as hard as any problem from the class NP of nondeterministic polynomial time decision problems~\cite{Blondel00}. In particular, NP is the class of all decision problems for which a candidate solution can be certified in polynomial time. 
Under this definition, we have the following result.
\begin{theorem}\label{thm:NPhard}
With  $\ell$ as in~\eqref{eq:loss}, Problem~\ref{pb:min} is NP-hard.
\end{theorem}

The proof is a direct consequence of the NP-completeness of the following decision form of Problem~\ref{pb:min}, where an NP-complete problem is one that is both NP-hard and in NP.
\begin{problem}[Decision form of switching regression]\label{pb:decision}
Given a data set $\{(\g x_i, y_i)\}_{i=1}^N \in( \QR^d\times \QR)^N$, an integer $n\in[2,N/d]$ and a threshold $\epsilon\geq 0$, decide whether there is a set of vectors $\{\g w_j\}_{j=1}^n\subset \QR^d$ and a labeling $\g q\in\Q^N$ such that
\begin{equation}\label{eq:decision}
	\frac{1}{N} \sum_{i=1}^N   \ell (y_i- \g w_{q_i}^T\g x_i )\leq \epsilon .
\end{equation}
\end{problem}
We show the completeness of Problem~\ref{pb:decision} by a reduction from the partition problem, known to be NP-complete~\cite{Garey79}.
\begin{problem}[Partition]\label{pb:partition}
Given a multiset (a set with possibly multiple instances of its elements) of $d$ positive integers, $S= \{s_1,\dots,s_d\}$, decide  
whether there is a multisubset $S_1\subset S$ such that
$$
	\sum_{s_i \in S_1} s_i = \sum_{s_i\in S\setminus S_1} s_i .
$$
\end{problem}

\begin{proposition}\label{prop:npcomplete}
Problem~\ref{pb:decision} is NP-complete.
\end{proposition}
\begin{proof}
Since given a candidate solution $\left(\{\g w_j\}_{j=1}^n, \g q\right)$ the condition~\eqref{eq:decision} can be verified in polynomial time, Problem~\ref{pb:decision} is in NP. Then, the proof of its NP-completeness proceeds by showing that the Partition Problem~\ref{pb:partition} has an affirmative answer if and only if a particular instance of Problem~\ref{pb:decision} has an affirmative answer. 

Given an instance of Problem~\ref{pb:partition}, build an instance of Problem~\ref{pb:decision} with $n=2$, $N=2d+1$, $\epsilon=0$ and a data set such that
$$
	(\g x_i, y_i) = \begin{cases}
		(s_i \g e_i,\ s_i) ,& \mbox{if } 1\leq i \leq d \\
		(s_{i-d} \g e_{i-d},\ 0) ,& \mbox{if } d< i \leq 2d \\
		\left(\g s = \sum_{k=1}^d s_k \g e_k\ ,\ \frac{1}{2}\sum_{k=1}^{d} s_k\right), & \mbox{if } i=2d+1,
	\end{cases}
$$
where $\g e_k$ is the $k$th unit vector of the canonical basis for $\QR^d$. 
If Problem~\ref{pb:partition} has an affirmative answer, let $I_1$ be the set of indexes of the elements of $S$ in $S_1$ and $I_2$ the set of indexes of the elements of $S$ not in $S_1$. Then we can set $\g w_1 = \sum_{i\in I_1} \g e_i$ and $\g w_2 = \sum_{i\in I_2} \g e_i$, which gives 
$$
	\g w_1^T\g x_i = \begin{cases}
		s_i = y_i ,& \mbox{if } i\leq d \mbox{ and } i\in I_1  \\
		0 ,& \mbox{if } i\leq d \mbox{ and }  i\in I_2 \\
		s_{i-d} = y_i ,& \mbox{if } i > d \mbox{ and } i-d \in I_1  \\		
		0 ,& \mbox{if } i>d \mbox{ and } i-d \in I_2 \\
		\sum_{k\in I_1} s_k = \frac{1}{2}\sum_{k=1}^{d} s_k =y_i, & \mbox{if } i=2d+1
	\end{cases}
$$
and
$$
	\g w_2^T \g x_i  = \begin{cases}
		0 ,& \mbox{if }  i\leq d \mbox{ and } i\in I_1 \\
		s_i = y_i ,& \mbox{if } i \leq d \mbox{ and } i\in I_2 \\
		0 ,& \mbox{if }  i > d \mbox{ and } i-d \in I_1  \\		
		s_{i-d} = y_i ,& \mbox{if } i > d \mbox{ and } i-d \in I_2 \\
		\sum_{k\in I_2} s_k =\frac{1}{2}\sum_{k=1}^{d} s_k = y_i, & \mbox{if } i=2d+1.
	\end{cases}
$$
Therefore, for all points, either $\g w_1^T\g x_i = y_i$ or $\g w_2^T\g x_i = y_i$, and~\eqref{eq:decision} holds with $\g q$ set as in~\eqref{eq:classif}, yielding an affirmative answer for Problem~\ref{pb:decision}. 

Assume now that Problem~\ref{pb:decision} has an affirmative answer with some $\{\g w_j\}_{j=1}^n$. Then, with $\epsilon=0$, the positivity of the loss function implies $\ell (y_i- \g w_{q_i}^T\g x_i ) = 0$, $i=1,\dots,N$, which, by~\eqref{eq:loss}  
yields
\begin{equation}\label{eq:perfectfit}
	\g w_1^T\g x_i  = y_i \quad \mbox{or}\quad  \g w_2^T \g x_i = y_i,\quad i=1,\dots,2d+1.
\end{equation}
We can always assume that $s_i \neq 0$, since otherwise $s_i$ can be removed from the problem statement. Under this assumption, if $\g w_1^T\g x_i  = y_i$ for some $i\leq d$, then $w_{1i} = 1$ and $\g w_1^T\g x_{d+i}  = s_i \neq y_{d+i}$, which further implies $\g w_2^T\g x_{d+i}  = y_{d+i} = 0$ and $w_{2i} = 0$. Conversely, if  $\g w_2^T\g x_i  = y_i$ for some $i\leq d$, then $w_{2i} = 1$ and $w_{1i} = 0$. 
Therefore,~\eqref{eq:perfectfit} leads to $w_{1i} \in \{0,1\}$ and $w_{2i} = 1 - w_{1i}$, $i=1,\dots,d$. 
In addition, recall that $y_{2d+1} = \frac{1}{2}\sum_{k=1}^{d} s_k$, such that, for $i=2d+1$,~\eqref{eq:perfectfit} yields at least one of the two equalities
\begin{align*}
	\g w_1^T \g x_{2d+1} = \sum_{k\in\{ i\leq d : w_{1i} = 1\} } s_k = \frac{1}{2}\sum_{k=1}^{d} s_k\\
	\g w_2^T \g x_{2d+1} = \sum_{k\in\{ i\leq d : w_{1i} = 0\} } s_k = \frac{1}{2}\sum_{k=1}^{d} s_k,
\end{align*}
and a partition corresponding to an affirmative answer for Problem~\ref{pb:partition} is given by $S_1 = \{s_i : w_{1i} = 1,\ i \leq d\}$. 
\end{proof} 

\begin{remark}
The proof of Proposition~\ref{prop:npcomplete} involves an instance of Problem~\ref{pb:decision} with $\epsilon=0$, i.e., a noiseless switching regression problem. However, it can be adapted to the restriction of Problem~\ref{pb:decision} to instances with $\epsilon >0$, implying the NP-hardness of Problem~\ref{pb:min} even when excluding noiseless instances.
\end{remark}

\section{Polynomial time complexity wrt. $N$}
\label{sec:exact}

We now turn to the analysis of the computational complexity of Problem~\ref{pb:min}  wrt. the number of data $N$, i.e., for fixed $n$ and data dimension $d$, under the following assumptions. 
\begin{assumption}\label{ass:generalposition}
The points $\{\g x_i\}_{i=1}^N$ are in general position, i.e., no hyperplane of $\QR^d$ contains more than $d$ points. Furthermore, the points $\g z_i=[\g x_i^T, y_i]^T$ are also in general position in $\QR^{d+1}$.
\end{assumption}
\begin{assumption}\label{ass:subpb}
Given $\{(\g x_i, y_i)\}_{i=1}^N \in (\QR^d\times\QR)^N$, the problem $\min_{\g w\in\QR^{d} } \sum_{i=1}^N \ell (y_i - \g w^T \g x_i)$ has a polynomial time complexity $T(N)$ for any fixed integer $d\geq 1$.
\end{assumption}
\begin{theorem}\label{thm:poly}
Under Assumptions~\ref{ass:generalposition}--\ref{ass:subpb}, for given integers $d$ and $n$, the time complexity of Problem~\ref{pb:min} is no more than polynomial in the number of data $N$ and in the order of  $T(N)\O(N^{2dn(n-1)})$.
\end{theorem}
Theorem~\ref{thm:poly} is a direct consequence of the existence of an exact algorithm that solves the problem in polynomial time, ensured by Corollary~\ref{col:exact} at the end of this section. This algorithm relies on the enumeration of all classifications consistent with~\eqref{eq:classif}, which we prove to be in a number polynomial in $N$ below.

For this, we will use results on the enumeration of all possible {\em linear} classifications of a set of $N$ points. In the binary case with two categories, a linear classification is one that can be produced by a separating hyperplane dividing the space in two halfspaces. It is shown in \cite{Lauer15} that the number of different linear classifications of $N$ points is on the order of $\O(N^d)$ in $\QR^d$ and that these can be constructed efficiently.    
Here, we use an adaptation of these results for linear classifiers, while \cite{Lauer15} focused on affine classifiers. This minor difference amounts to the removal of a degree of freedom by forcing the hyperplane to pass through the origin. Since the results of \cite{Lauer15} are based on hyperplanes passing through sets of $d$ points, they can be directly extended to linear classifiers by choosing one of these points to be the origin. Thus, we state the following without proof. 

\begin{proposition}[Adapted from Theorem~3 in \cite{Lauer15}]\label{prop:enum}
Let
$$
	\H_S=\left\{ \g b \in \{-1,+1\}^N : b_i= \sign(\g h^T \g x_i),\ i\in[N],\ \g h\in\QR^d \right\}
$$
denote the set of binary linear classifications of $N$ points $S=\{\g x_i\}_{i=1}^N \subset \QR^d$. 
Then, for any $N>d$, its cardinality is bounded as  
$
	\sup_{S\in(\QR^d)^N} |\H_S| \leq 2^{d} \binom{N}{d-1}
$
and, for any set $S$ of $N$ points in general position, there is an algorithm that builds $\H_S$ 
in $\O\left( 2^{d}\binom{N}{d-1}\right)$ iterations.
\end{proposition}

In PWA regression, the modes are typically assumed to be linearly separable in the regression space $\QR^d$ and results in the flavor of Proposition~\ref{prop:enum} can readily be applied to find the optimal classification of the data points \cite{Lauer15}. 
In switching regression, the mode sequence $\{q_i\}$ is arbitrary and we cannot assume the modes to be linearly separable.  
However, the groups of data pairs $(\g x_i, y_i)$ associated to different linear models can be ``linearly separated" in some sense. 
More precisely, we will show that the classification rule \eqref{eq:classif} implicitly entails a combination of two linear classifiers: one applying to the points $\g z_i = [\g x_i^T, y_i]^T$ in $\QR^{d+1}$ and another one applying to the regression vectors $\g x_i$ in $\QR^d$. The equivalence between~\eqref{eq:classif} and these linear classifiers will hold for all points with index not in 
\begin{align}\label{eq:setE}
E = \{ i \in [N] :&\ \exists (j,k) \in \Q^2,\ j\neq k, |y_i-{\g w_j}^T \g x_i|=|y_i-{\g w_k}^T \g x_i| \}, 
\end{align}
whose cardinality is bounded by the following lemma. 
\begin{lemma}\label{lem:setE}
Let $E$ be defined as in~\eqref{eq:setE}. Under Assumption~\ref{ass:generalposition}, $|E| \leq (2d+1)n(n-1)/2$.
\end{lemma} 
\begin{proof}
Let us define, for all $(j,k)$ such that $1\leq j < k \leq n$, the sets $I_{jk} = \{i\in[N] : {\g w_j}^T \g x_i = {\g w_k}^T \g x_i \}$ and $M_{jk} = \{ i\in[N] : y_i-{\g w_j}^T \g x_i= - (y_i-{\g w_k}^T \g x_i) \}$. Then, we have $E = \bigcup_{1\leq j < k \leq n} I_{jk} \cup M_{jk}$. 
Since ${\g w_j}^T \g x_i = {\g w_k}^T \g x_i \Leftrightarrow  (\g w_j - \g w_k)^T \g x_i = 0$, all points $\g x_i$ with $i\in I_{jk}$ must lie on a hyperplane of $\QR^d$, and under Assumption~\ref{ass:generalposition} we have $|I_{jk}| \leq d$. Similarly, since $y_i-{\g w_j}^T \g x_i= - (y_i-{\g w_k}^T \g x_i) \Leftrightarrow  y_i-(\g w_j - \g w_k)^T \g x_i / 2 = 0$, all points $\g z_i= [\g x_i^T, y_i]^T$ with $i\in M_{jk}$ must lie on a hyperplane of $\QR^{d+1}$, and Assumption~\ref{ass:generalposition} implies that $|M_{jk}| \leq d+1$. Hence, $|E| \leq \sum_{1\leq j < k \leq n} |I_{jk}| + |M_{jk}| \leq \sum_{1\leq j < k \leq n} 2d+1 \leq (2d+1)n(n-1)/2$.
\end{proof}

\begin{proposition}\label{prop:equivclassif}
Given a set of parameter vectors $\{\g w_j \}_{j=1}^n\subset\QR^d$, let $E$ be defined as in~\eqref{eq:setE}. Then, for all $i\notin E$, the classification rule~\eqref{eq:classif} with a loss function satisfying~\eqref{eq:loss} is equivalent to the classification rule 
\begin{equation}\label{eq:majorityvote}
	q_i = \argmax_{j\in\Q}  \sum_{k=1}^{j-1}\I{c_{kj}( \g x_i, y_i) = -1}  + \sum_{k=j+1}^n \I{c_{jk}( \g x_i, y_i) = +1}
\end{equation}
implementing a majority vote over a set of $n(n-1)/2$ pairwise classifiers $\{c_{jk}\}_{1\leq j<k\leq n}$ of $\QR^{d+1}$, where each $c_{jk}$ is a product of binary linear classifiers defined as 
$$
	\forall (\g x, y)\in\QR^d\times\QR,\ c_{jk}( \g x, y) = g_{jk}(\g z) h_{jk}(\g x),\quad 1 \leq j < k \leq n,
$$
with linear classifiers respectively operating in $\QR^{d+1}$ and $\QR^d$ as
\begin{align*}
	\forall \g z \in \QR^{d+1},\ & g_{jk}(\g z) = \sign( [-\overline{\g w}_{jk}^T,\ 1]^T \g z) ,\quad 1 \leq j < k \leq n,\\
	\forall \g x \in \QR^d,\ & h_{jk}(\g x ) = \sign(\tilde{\g w}_{jk}^T \g x) ,\quad 1 \leq j < k \leq n,
\end{align*}
where $\overline{\g w}_{jk} = (\g w_j + \g w_k ) /2$ and $\tilde{\g w}_{jk} = \g w_j - \g w_k$. 
\end{proposition}
\begin{proof}
Using the properties of the loss function~\eqref{eq:loss}, the classification rule \eqref{eq:classif} can be rewritten for any 
$i\notin E$ as
\begin{align}
 	  & q_i = \argmin_{k\in\Q} \ell(y_i - \g w_k^T \g x_i ) \nonumber \\ 
 	 \Leftrightarrow  \quad &  \forall k\in\Q\setminus\{q_i\},\  |y_i - \g w_{q_i}^T \g x_i | < |y_i - \g w_k^T \g x_i | .\label{eq:eq1}
\end{align}
For any triplet $(a,b,y)\in \QR^3$, we have $|y-a| < |y-b|$ if and only if $(a < b \wedge y< (a+b)/2 )$ or $(a> b \wedge y> (a+b)/2)$, i.e.,  $|y-a| < |y-b| \Leftrightarrow (y-(a+b)/2)(a-b)> 0$. Thus, with the notations of Proposition~\ref{prop:equivclassif}, \eqref{eq:eq1} is equivalent to
\begin{align}
   &\quad \forall k\in\Q\setminus\{q_i\},\ (y - \overline{\g w}_{q_ik}^T \g x_i)  \tilde{\g w}_{q_ik}^T \g x_i  > 0	\nonumber\\
  \Leftrightarrow & \quad \forall k\in\Q\setminus\{q_i\},\  g_{q_ik}(\g z_i)h_{q_ik}(\g x_i ) = +1 \nonumber\\
	\Leftrightarrow & \quad \sum_{k\in\Q\setminus\{q_i\}} \I{c_{q_ik}( \g x_i, y_i) = +1} = n-1 .\label{eq:eq2}
\end{align}
In addition, for all $(j,k) \in \{1,\dots,n\}^2$, $j\neq k$, we have $\overline{\g w}_{jk}=\overline{\g w}_{kj}$ and $\tilde{\g w}_{jk} = -\tilde{\g w}_{kj}$, so that $c_{jk}(\g x, y) = - g_{kj}(\g z ) h_{kj}(\g x) = -c_{kj}(\g x, y)$. Thus, the classification is entirely determined by a set of $n(n-1)/2$ pairs of linear classifiers $g_{jk}$ and $h_{jk}$, $1\leq j < k \leq n$, and~\eqref{eq:eq2} is equivalent to
\begin{align*} 
  	 S(q_i) \triangleq \sum_{k=1}^{q_i-1}\I{c_{kq_i}( \g x_i, y_i) = -1}  + \sum_{k=q_i+1}^n \I{c_{q_i k}( \g x_i, y_i) = +1} &= n-1 .
\end{align*}
Given that $\max_{j\in\Q} S(j) \leq n-1$, 
we obtain that $q_i \in \argmax_{j\in\Q} S(j)$. In addition, for all $q\neq q_i$, $q\in\argmax_{j\in\Q} S(j)$ implies $S(q) = S(q_i) = n-1$, which implies by working the equivalences above backward that $q\in\argmin_{k\in\Q} \ell(y_i - \g w_k^T \g x_i )$. However, this is not possible, since for all $i\notin E$, $\argmin_{k\in\Q} \ell(y_i - \g w_k^T \g x_i )$ is a singleton. 
Thus, $q_i$ is the only element in $\argmax_{j\in\Q} S(j)$ as claimed in~\eqref{eq:majorityvote}. 
\end{proof}

Thanks to Proposition~\ref{prop:equivclassif}, a bound on the number of classifications of $N$ points by~\eqref{eq:classif} can be formed from the product of the number of classifications of the $\g z_i$'s and of the $\g x_i$'s.
\begin{theorem}\label{thm:minclassif}
Let us define the set of minimum-error classifications of a data set $S=\{(\g x_i, y_i)\}_{i=1}^N\subset\QR^d\times \QR$ as
\begin{align*}
	\Q_S = \left\{ \g q \in \Q^N :\ q_i = \argmin_{j\in\Q} \ell(y_i - \g w_j^T \g x_i ),\ i\in[N], \g w_j \in \QR^d \right\} .
\end{align*}
Then, under Assumption~\ref{ass:generalposition}, 
$$
	\Pi(N) \triangleq \sup_{S\in (\QR^d\times\QR)^N} |\Q_S| = \O(N^{2dn(n-1)})
$$
and $\Q_S$ can be computed in $\O(N^{2dn(n-1)})$ time.
\end{theorem}

\begin{proof}
By Proposition 3, any $\g q\in\Q_S$ must result from a set of $n(n-1)/2$ products of linear classifications, possibly altered for all $i\in E$. Assuming $|E|=l$, $n^l$ such altered classifications can be generated from a ``base classification" given by \eqref{eq:majorityvote}. The number of base classifications returned by \eqref{eq:majorityvote} is bounded by the number of classifications generated by $n(n-1)/2$ pairs of linear classifiers $(g_{jk}, h_{jk})$, $1\leq j<k\leq n$, on $N-l$ points. 
Applying Proposition~\ref{prop:enum} twice, once in dimension $d+1$ to bound the number of classifications returned by a $g_{jk}$ and once in dimension $d$ to bound the number of classifications returned by a $h_{jk}$, we can upper bound the number of classifications returned by a product classifier $c_{jk}$ by
$
	2^{2d+1}\binom{N-l}{d} \binom{N-l}{d-1} = 
	 \O(N^{2d-1})
$ 
and obtain an algorithm of a similar time complexity to compute these base classifications. 
With $p=(2d-1)n(n-1)/2$, this leads to  
$ n^l\O(N^p)$ classifications for each possible set $E$. 
Given an upper bound $L$ on $|E|$, summing over all $E$ of cardinality $l\leq L$ yields
\begin{align*}
	\Pi(N) &\leq  \sum_{l=0}^L \begin{pmatrix}N\\l\end{pmatrix} n^l \O(N^{p}) 
	 \leq \sum_{l=0}^L  \O(N^l) n^l \O(N^p) \\
	&< L  \O(N^L) n^L \O(N^p) = \O(N^{L+p}).
\end{align*}
Combining this with the value of the bound $L$ given by Lemma~\ref{lem:setE} yields the desired result.
\end{proof}

Theorem~\ref{thm:minclassif} implies the following for switching regression. 

\begin{corollary}\label{col:exact} 
Under Assumptions~\ref{ass:generalposition}--\ref{ass:subpb}, there is an algorithm that exactly solves Problem~\ref{pb:min} in $T(N)\O(N^{2dn(n-1)})$ time. 
\end{corollary}
\begin{proof}
By Theorem~\ref{thm:minclassif}, an algorithm can generate all classifications $\g q\in \Q_S$ that are consistent with~\eqref{eq:classif} and thus visit the optimal $\g q$ for Problem~\ref{pb:min} in $\O(N^{2dn(n-1)})$ iterations. Following the discussion of Sect.~\ref{sec:intro}, the algorithm can compute the optimal $\g w_j$'s for any $\g q$ by solving~\eqref{eq:wj}, and thus find the global solution to Problem~\ref{pb:min}. 
Under Assumption~\ref{ass:subpb}, the cost of solving the $n$ subproblems in~\eqref{eq:wj} at each iteration is $nT(N)$, leading to the claimed time complexity for a fixed $n$.  
\end{proof}

\begin{remark}
The algorithm of Corollary~\ref{col:exact} has a polynomial complexity in $N$ but an exponential one in $d$. Therefore, by using the reduction from Problem~\ref{pb:partition} to Problem~\ref{pb:decision} of the proof of Proposition~\ref{prop:npcomplete}, its complexity remains exponential in the size of Problem~\ref{pb:partition}, which is in line with the conjecture P$\neq$NP. 
\end{remark}
\begin{remark}
Assumption~\ref{ass:generalposition} clearly cannot hold without noise, since, in this case, the points $\g z_i$ precisely belong to the union of $n$ hyperplanes. 
However, in the noiseless case, an algorithm that runs in time polynomial in $N$ can be devised from the fact that each parameter vector $\g w_j$ can be determined from a subset of $d$ points from the same mode. Assuming that such a subset exists for all modes, it suffices to find these subsets.
A straightforward strategy is to consider all subsets of $d$ data points among $N$ for the first mode, $d$ points among $N-d$ for the second mode and so on. The number of collections of such disjoint subsets is 
$
	\prod_{k=0}^{n-1} \binom{N - kd}{d} < \binom{N}{d}^n < \left(\frac{N}{d!}\right)^{nd} = \O(N^{nd})
$ 
and thus polynomial in $N$. Since testing one of these collections amounts to solving $n$ linear systems in constant time $\O(nd^3)$, the resulting algorithm runs in time polynomial in~$N$.
\end{remark}

\section{Conclusions}
\label{sec:discuss} 

The paper showed that globally minimizing the error of a switching linear regression model is NP-hard, but also that, for fixed data dimension and number of modes, an exact algorithm with polynomial complexity in the number of data exists. This algorithm has an exponential complexity wrt. the data dimension and the number of modes, which strongly limits its practical applicability. Yet, the existence of an exact algorithm with polynomial complexity in the dimension is unlikely given the NP-hardness of the problem and the fact that it holds also with a fixed number of modes $n=2$. Instead, future work will focus on polynomial-time approximation schemes.

\end{document}